\documentclass[twoside,11pt]{article}

\usepackage[accepted]{melba}
\usepackage{graphicx,subcaption}
\usepackage{amsmath,amssymb,amsfonts}
\usepackage{bbm}
\usepackage{float}
\usepackage{algorithm,algpseudocode}
\algnewcommand{\Inputs}[1]{%
  \State \textbf{Inputs:}
  \Statex \hspace*{\algorithmicindent}\parbox[t]{.8\linewidth}{\raggedright #1}
}
\algnewcommand{\Initialize}[1]{%
  \State \textbf{Initialize:}
  \Statex \hspace*{\algorithmicindent}\parbox[t]{.8\linewidth}{\raggedright #1}
}
\algnewcommand{\Outputs}[1]{%
  \State \textbf{Outputs:}
  \Statex \hspace*{\algorithmicindent}\parbox[t]{.8\linewidth}{\raggedright #1}
}
\def\ka{\kappa}

\def\Ga{\Gamma}

\def\ver{\on{\ver}}

\def\Id{\on{Id}}
\def\Diff{\on{Diff}}

\def\x{\times}

\def\R{{\mathbb R}}

\def\exp{\operatorname{exp}}

\let\on=\operatorname

\def\Tr{\on{Tr}}
\def\vol{\on{vol}}
\def\dist{\on{dist}}
\def\Met{\on{Met}}

\usepackage{todonotes}

\usepackage{tikz} 
\usetikzlibrary{cd}
\makeatletter \def\bign#1{\mathclose{\hbox{$\left#1\vbox to8.5\p@{}\right.\n@space$}}\mathopen{}} \def\Bign#1{\mathclose{\hbox{$\left#1\vbox to11.5\p@{}\right.\n@space$}}\mathopen{}} \def\biggn#1{\mathclose{\hbox{$\left#1\vbox to14.5\p@{}\right.\n@space$}}\mathopen{}} \def\Biggn#1{\mathclose{\hbox{$\left#1\vbox to17.5\p@{}\right.\n@space$}}\mathopen{}} \makeatother

\graphicspath{{figures/}}

\begin{document}
\title{Integrated Construction of Multimodal Atlases with Structural Connectomes in the Space of Riemannian Metrics}

\melbaheading{2022:016}{https://www.melba-journal.org/papers/2022:016.html}{2022}{1-25}{09/2021}{06/2022}{Campbell, Dai, Su, Bauer, Fletcher and Joshi}{Information Processing in Medical Imaging (IPMI) 2021}{Aasa Feragen, Stefan Sommer, Julia Schnabel, Mads Nielsen}

\ShortHeadings{Integrated Construction of Multimodal Atlases}{K. M. Campbell et al.}
%

\author{\name Kristen M. Campbell \email kris@sci.utah.edu \\
\addr Scientific Computing and Imaging Institute, University of Utah, Salt Lake City, UT, USA \\
\addr School of Computing, University of Utah, Salt Lake City, UT, USA 
\AND
\name Haocheng Dai \email hdai@sci.utah.edu\\
\addr Scientific Computing and Imaging Institute, University of Utah, Salt Lake City, UT, USA \\
\addr School of Computing, University of Utah, Salt Lake City, UT, USA 
\AND
\name Zhe Su \email zsu20@g.ucla.edu \\
\addr Department of Neurology, University of California Los Angeles, Los Angeles, CA, USA
\AND
\name Martin Bauer \email bauer@math.fsu.edu\\
\addr Department of Mathematics, Florida State University, Tallahassee, FL, USA 
\AND
\name P. Thomas Fletcher \email ptf8v@virginia.edu \\
\addr Department of Electrical \& Computer Engineering, University of Virginia, Charlottesville, VA, USA \\
\addr Department of Computer Science, University of Virginia, Charlottesville, VA, USA
\AND
\name Sarang C. Joshi \email sjoshi@sci.utah.edu \\
\addr Scientific Computing and Imaging Institute, University of Utah, Salt Lake City, UT, USA \\
\addr Department of Bioengineering, University of Utah, Salt Lake City, UT, USA}

\maketitle              

\begin{abstract}%
The structural network of the brain, or structural connectome, can be represented by fiber bundles generated by a variety of tractography methods.  While such methods give qualitative insights into brain structure, there is controversy over whether they can provide quantitative information, especially at the population level.  In order to enable population-level statistical analysis of the structural connectome, we propose representing a connectome as a Riemannian metric, which is a point on an infinite-dimensional manifold.  We equip this manifold with the Ebin metric, a natural metric structure for this space, to get a Riemannian manifold along with its associated geometric properties.  We then use this Riemannian framework to apply object-oriented statistical analysis to define an atlas as the Fr\'echet mean of a population of Riemannian metrics.  This formulation ties into the existing framework for diffeomorphic construction of image atlases, allowing us to construct a multimodal atlas by simultaneously integrating complementary white matter structure details from DWMRI and cortical details from T1-weighted MRI.  We illustrate our framework with 2D data examples of connectome registration and atlas formation. Finally, we build an example 3D multimodal atlas using T1 images and connectomes derived from diffusion tensors estimated from a subset of subjects from the Human Connectome Project.
\end{abstract}

\section{Introduction}
Tractography is one way to represent a structural connectome, or structural network of a brain, which consists of brain regions that are physically connected by a network of neuronal bundles that make up the white matter of that brain.  It is not currently possible to image individual neurons in a living brain non-invasively.  Instead, we can use diffusion-weighted magnetic resonance images (DWMRI) to infer the pathways of coherent bundles of neurons that cross through other bundles in white matter and connect with end points in gray matter.
There are a variety of tractography algorithms used to infer these white matter pathways, ranging from local methods that integrate local orientation information to form individual streamlines to global methods that estimate all fiber tracts simultaneously. Deterministic~\citep{basser2000vivo} and probabilistic~\citep{behrens2003characterization} streamline integration methods are easy to compute, but they suffer from accumulation of local orientation errors leading to tract reconstructions biased toward shorter and straighter tracts. Conversely, global estimation methods~\citep{jbabdi2007bayesian,christiaens2015global} incorporate prior anatomical knowledge while ensuring that tractography is consistent with the underlying data. However, these global methods have convergence issues, sensitivity to initialization and priors, and a tendency to have estimated tracts end in the middle of white matter regions.  Biases in tract reconstruction introduced by either local or global methods affect the accuracy of quantitative measures such as track density and connection strength.  

Geodesic tractography algorithms, first introduced by~\cite{o2002new}, use a combination of local and global information to determine tracts by formulating white matter pathways as geodesic curves under a Riemannian metric derived from the DWMRI data. In the original work, \citet{o2002new} use the inverse diffusion tensor as the Riemannian metric. As described by \citet{lenglet2004}, the inverse diffusion tensor metric has a connection to Brownian motion through the Laplace Beltrami operator on the resulting Riemannian manifold. The inverse tensor metric favors paths that follow the principal eigenvector of the diffusion tensor, as this is the locally optimal direction to move. However, geodesic curves for this metric do not consistently follow the principal eigenvector of the diffusion tensor, and tend to be overly straight in regions where the white matter fibers are bending, losing association with the underlying anatomy in these areas. This problem has been addressed by several strategies to improve the adherence of geodesics to the white matter geometry, including ``sharpening'' the inverse diffusion tensor~\citep{fletcher2007}, using the adjugate of the diffusion tensor~\citep{fuster2016}, and using a conformal metric~\citep{hao2014improved}. One advantage to the geodesic tractography formulation is that we can do uncertainty and confidence interval analysis of the tractographies following~\citep{sengers2021}.

These tractography techniques help to give insight into the structure of a single brain, but it remains an open challenge to quantitatively measure how these structural connectome pathways vary in a population.
To do such a population analysis, we first need a common frame of reference, or atlas space, where spatial differences between subjects' white matter can be measured.  

Initially, white matter atlases were constructed by aligning DWMRI to an anatomical template, then transforming and averaging the associated tensor field or distribution function field. For example, \citet{mori2008dti} construct a diffusion tensor imaging (DTI) atlas by registering the DWMRI of multiple subjects to a standardized anatomical template. They build the DTI atlas by transforming the diffusion tensors for each subject~\citep{alexander2001spatial} and then taking the Euclidean average of the transformed diffusion tensors at each voxel. This approach does not use the white matter directionality information encoded in the diffusion images during the registration. It also suffers from the fact that the Euclidean average of diffusion tensors does not take into account the directionality and tends to be fatter (i.e., less anisotropic) than the input tensors~\citep{fletcher2007riemannian}. Another approach by~\citet{yeh2018population} is to register $q$-space diffusion images into an anatomical template and estimate the spin distribution function (SDF) at each voxel in the template. Then the SDFs are averaged on a per-voxel basis. While this method does take into account the directionality of the white matter in a local neighborhood, it does not take into account consistency of long-range white matter connections.  These methods both rely on low-accuracy registration to the anatomical template image, because at the time high-accuracy diffeomorphic registrations could not be used as they did not preserve the continuity of pathway directions across voxels~\citep{yeh2018population}.  

The field then focused on constructing atlases directly from tractography.  These methods  are adept at finding well-known tracts, but they generally require significant expert curation to remove the many false-positive fibers and tracts they create~\citep{zhang2018anatomically}.  Additionally, variation of population-level tractography characteristics like fiber density of tracts is as likely to be reflective of the tractography process chosen as it is of the underlying physical white matter.  These and other biases in tractography quantification are well-characterized by~\citet{jeurissen2019diffusion}.

In order to get a more complete anatomical atlas, ~\citet{toga2006towards} argue for the integrated derivation of multimodal atlases using techniques such as spatial normalization to produce more comprehensive atlases.  Some attempts have been made to create a multimodal population atlas from T1 and DWMRI images. \citet{gupta2016framework} rigidly register the DWMRI into the same space as the T1 template image and then apply the same deformations used to create the T1 template to the DWMRI before estimating the diffusion tensors from the transformed DWMRI.  This approach puts the DWMRI and T1 images in the same space, but does not use the white matter structure to inform the creation of the T1 atlas. Other previous work on multichannel registration of diffusion and T1 MRI include \citet{avants2007multivariate}, who use a Euclidean image match metric on the diffusion tensors, and \citet{uus2020multi}, who use local angular correlation on the orientation distribution functions (ODFs). Like the white matter-only atlases approaches discussed above, these methods only take local information into consideration and thus do not preserve the consistency of long-range white matter connections. They do demonstrate that registration quality is improved over single channel registration when complementary channels are combined in the objective function.

We are motivated to find an approach that can preserve the best aspects of these atlas and tractography methods, while mitigating their weaknesses.  Specifically, we want to create a white matter pathway atlas that preserves local orientations and other anatomical information while maintaining the continuity and integrity of long-range connections.  We then want a method to bring subjects into that atlas space to enable statistical quantification of both structural connectivity and geometric variability of white matter structure across a population. 

To meet these goals, we describe the metric matching framework presented by~\citet{campbell2021structural} in more detail and then extend it by combining diffeomorphic metric matching with diffeomorphic image matching to enable the construction of both a multimodal white and gray matter atlas simultaneously for the first time.  This formulation preserves geodesics transformed by diffeomorphisms, which meets our objective to use local diffusion data and maintain the integrity of long-range connectomics as inferred by tractography~\citep{cheng2015tractography}.  We demonstrate that long-range connections are preserved by performing geodesic tractography on the resulting atlas.

\subsection{Contributions of the Article}
In this paper, we contribute a mathematical framework for diffeomorphic metric matching that is compatible with existing image matching frameworks to enable the creation of integrated multimodal atlases.  We start by using the concept from geodesic tractography to represent connectome fibers as geodesics of a metric, that is, each brain's white matter structure is represented as a point on the infinite-dimensional manifold of Riemannian metrics.  This manifold is then equipped with the diffeomorphism-invariant Ebin metric to compute distances and geodesics between connectomes.   We explain how this Riemannian manifold is the foundation for the algorithm for diffeomorphic metric registration of structural connectomes and the statistical groupwise metric atlas estimation algorithm. We then extend this model by including an image matching term in those algorithms.  Finally, we simultaneously estimate an integrated multimodal white matter pathway and T1 MRI-based image atlas for the first time. This article is an extended version of the  Information Processing in Medical Imaging (IPMI) conference paper~\citep{campbell2021structural}, where we expand the results of the conference proceedings in several major directions. Most importantly, the joint white matter pathway and T1 MRI-based image atlas model is newly introduced and, in contrast to~\cite{campbell2021structural} which only contained 2D examples, we now present 3D atlas construction examples.

\section{Structural Connectomes as Riemannian Metrics}\label{sec:structasmet}

 In the white matter of the brain, the diffusion of water is restricted perpendicular to the
 direction of the axons. Diffusion-weighted MRI measures the
 microscopic diffusion of water in multiple directions at every voxel in a 3D volume. Thus, the
 directionality of white matter in the brain can be locally inferred. 
Traditionally, global connections of the white matter have been estimated by a procedure called {\it tractography}, which numerically computes integral curves of the vector field formed by the most likely direction of fiber tracts at each point. DTI models anisotropic water diffusion with a 3x3 symmetric positive definite tensor, $D(x)$, at each voxel, $x\in M$, where the manifold $M \subset \mathbb{R}^3$ is the image domain. The principal eigenvector of $D(x)$ is aligned with the direction of the strongest diffusion.  

Riemannian metrics that represent connectomics of a subject have been developed in diffusion imaging by ~\cite{o2002new} and include the inverse-tensor metric $\tilde g = D(x)^{-1}$. However, the geodesics associated with the inverse-tensor metric tend to deviate from the principal eigenvector directions and take straighter paths through areas of high curvature.

In this work we build on the algorithm developed by~\cite{hao2014improved}, which estimates a spatially-varying function, $\alpha(x) : M \to \mathbb{R}$, that modulates the inverse-tensor metric to create a locally-adaptive Riemannian metric, $g_\alpha = e^{\alpha(x)}\tilde g$.
We briefly describe the method here for completeness but refer the reader to~\cite{hao2014improved} for details. This adaptive \emph{connectome metric}, $g_\alpha$, is conformally equivalent to the inverse-tensor metric and is better at capturing the global connectomics, particularly through regions of high curvature.  Figure~\ref{fig:metricestgeos} shows how well the geodesics of each metric match the integral curve of the vector field.  The connectome metric geodesics are very closely aligned with the integral curves.

The geodesic between two end-points, $p, q \in M$, associated with the inverse-tensor metric, $\tilde g(x) = D(x)^{-1}$, minimizes the energy functional, $\tilde E$, while the geodesic associated with the connectome metric, $g_\alpha(x) = e^{\alpha(x)}D(x)^{-1}$, minimizes the energy functional, $E_\alpha$.
\begin{equation}
   \begin{aligned}
   \tilde E(\gamma) = \int_0^1\langle T(t), T(t) \rangle_{\tilde g} dt,
   \end{aligned}  
   \qquad
   \begin{aligned}
   E_\alpha(\gamma) = \int_0^1e^{\alpha(x)}\langle T(t), T(t) \rangle_{\tilde g} dt,   
   \end{aligned}
\end{equation}
where $\gamma: [0, 1] \to M$, $\gamma(0) = p$, $\gamma(1) = q$, $T= \frac{d\gamma}{dt}$. 

Analyzing the variation of $E_\alpha$ leads to the geodesic equation, $\mathrm{grad}\, \alpha = 2 \nabla_T T$, where the Riemannian gradient of $\alpha$, $\mathrm{grad}\, \alpha = \tilde g^{-1}\bigl(\frac{\partial \alpha}{\partial x^1}, \frac{\partial \alpha}{\partial x^2}, \cdots, \frac{\partial \alpha}{\partial x^n}\bigr)$, and $\nabla_T T$ is the covariant derivative of $T$ along its integral curve.

To enforce the desired condition where the tangent vectors, $T$, of the geodesic match the vector field, $V$, of the unit principal eigenvectors of $D(x)$, we minimize the functional, $F(\alpha) = \int_M || \mathrm{grad}\, \alpha - 2 \nabla_V V ||^2_{\tilde g} dx$.  The equation for $\alpha$ that minimizes $F(\alpha)$ is 
\begin{equation}
\Delta_{\tilde g} \alpha = 2\, \mathrm{div}_{\tilde g} (\nabla_V V),
\label{eqn:poisson}
\end{equation}
where $\mathrm{div}$ and $\Delta$ are the Riemannian divergence and  Laplace-Beltrami operator.
We discretize the Poisson equation in Equation \eqref{eqn:poisson} using a second-order finite difference scheme that satisfies both the Neumann boundary conditions $\frac{\partial \alpha}{\partial \overrightarrow{n}}= \langle \mathrm{grad}\, \alpha, \overrightarrow{n} \rangle = \langle 2\nabla_V V, \overrightarrow{n}\rangle$ and the governing equation on the boundary.  We then solve for $\alpha$.

Note that we can use this method to match the geodesics of the connectome metric to other vector fields defining the tractogram, e.g., from higher-order diffusion models that can represent multiple fiber crossings in a voxel. In particular, for tractography based on fiber orientation distributions (FODs), we can use the techniques presented in \cite{nie2019topographic} to generate the vector field $V$.

\begin{figure}[h]
\centering
\includegraphics[width=\linewidth]{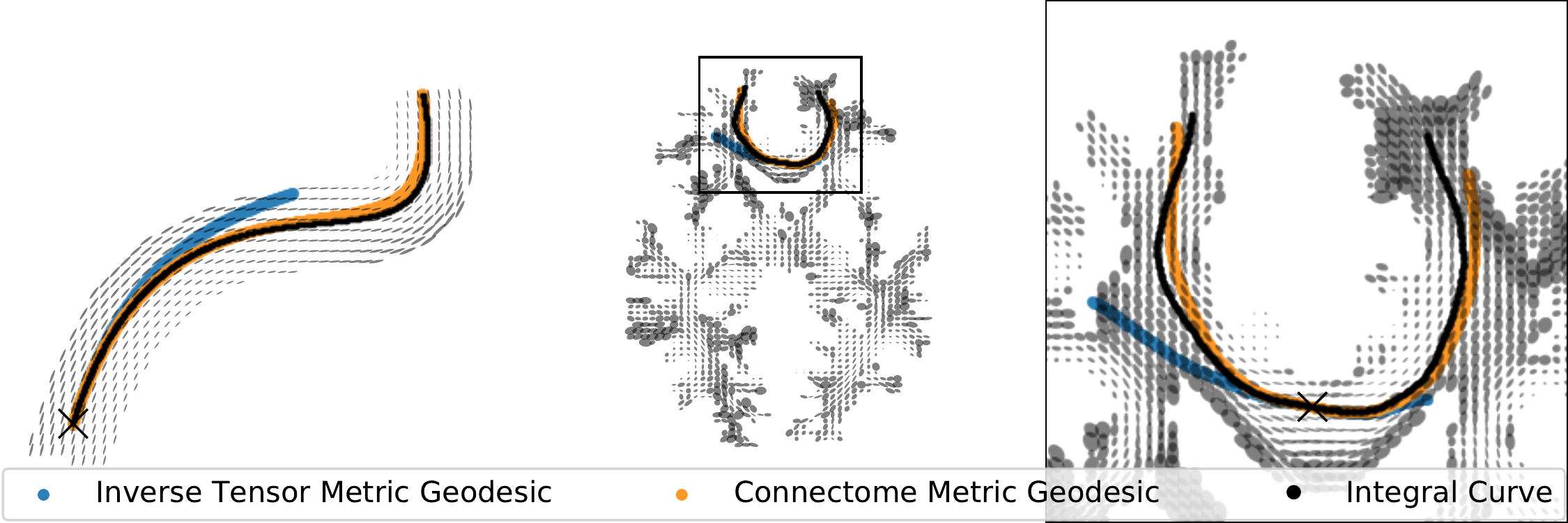}
\caption{A geodesic of the inverse-tensor metric (blue) and adaptive metric (orange), along with an integral curve (black) associated with the principal eigenvectors for a synthetic tensor field (left) and a subject's connectome metric from the Human Connectome Project (center).  Right shows a detailed view of the metric in the corpus callosum.}
\label{fig:metricestgeos}
\end{figure}

The fundamental result of the purposed method is the fact that a metric tensor field is estimated such that a given vector field is a geodesic vector field of the estimated Riemannian metric. It is in this sense that the estimated metric field captures the geometry of the white matter as inferred by tractoraphy. This metric field that captures the tractography will be the primary driving force for the registration algorithms. We will now develop registration and atlas construction algorithms that leverage the rich geometric structure of Manifold of all Riemannian Metrics.

\section{The Geometry of the Manifold of all Riemannian Metrics}\label{sec:riemet}
Once we have estimated a Riemannian metric for a human connectome, it is a point in the infinite-dimensional manifold of all Riemannian metrics, $\Met(M)$, where $M$ is the domain of the image. We will equip this space  with a diffeomorphism-invariant Riemannian metric, called the Ebin or DeWitt metric~\citep{Ebin1970a,DeWitt67}. 
The invariance of the infinite-dimensional metric under the group of diffeomorphisms  $\Diff(M)$ is a crucial property, as it guarantees the independence of an initial choice of coordinate system on the brain manifold. As we will base our statistical framework on this infinite-dimensional geometric structure, we will now give a detailed overview of its induced geometry on $\Met(M)$.

Let  $M$ be a smooth $n$-dimensional manifold; for our targeted applications $n$ will be two or three. We denote by $\Met(M)$ the space of all smooth Riemannian metrics on $M$, i.e., each element $g$ of the space $\Met(M)$ is a symmetric, positive-definite $0 \choose 2$ tensor field on $M$.  It is convenient to think of the elements of $M$ as being point-wise positive-definite sections of the bundle of symmetric two-tensors $S^2 T^\ast M$, i.e., smooth maps from $M$ with values in $S^2_+ T^\ast M$.  
Thus, the space $\Met(M)$  is an open subset of the linear space $\Ga(S^2 T^\ast M)$ of all smooth symmetric $0 \choose 2$ tensor fields and hence itself a smooth Fr\'echet-manifold \citep{Ebin1970a}. Furthermore, let $\Diff(M)$ denote the infinite-dimensional Lie group of all smooth diffeomorphisms of the manifold $M$. Elements of 
$\Diff(M)$ act as coordinate changes on the manifold $M$. This group acts on the space of metrics via pullback
\begin{align}\label{eq:diff_action_pullback}
\Met(M)\times\Diff(M)\to \Met(M), \qquad (g,\varphi)\mapsto \varphi^*g=g(T\varphi \cdot, T\varphi \cdot)\;.
\end{align}
In an analogous way we can define the pushforward action
\begin{align}\label{eq:diff_action}
\Met(M)\times\Diff(M)\to \Met(M), \qquad (g,\varphi)\mapsto \varphi_*g=\left(\varphi^{-1}\right)^{*}g
\end{align}

It is important to note that the geometries of the metrics $g$ and $\varphi^*g$ ($\varphi_* g$ resp.) are also related via $\varphi$. In particular, geodesics with respect to $g$ are mapped via $\varphi$ to geodesics with respect to $\varphi^*g$ (via $\varphi^{-1}$ for $\varphi_* g$, resp.). 
On the infinite-dimensional manifold $\Met(M)$, there exists a natural Riemannian metric:
the reparameterization-invariant $L^2$-metric. To define the metric, we need to first characterize the tangent space of the manifold of all metrics:  $\Met(M)$ is an open subset of  $\Ga(S^2 T^\ast M)$. Thus, every  
tangent vector $h$ is a smooth bilinear form $h: TM \x_M TM \to \R$ that can be equivalently interpreted as a map $TM \to T^\ast M$.  The $L^2$-metric is given by
\begin{align}\label{eq.Ebinmetric}
    G^{E}_g(h,k)=\int_M \on{Tr}\big(g^{-1} hg^{-1} k\big)\vol(g),
\end{align}
with $g \in \Met(M)$, $h,k\in T_g\Met(M)$ and $\vol(g)$ the induced volume density of the metric $g$.  This metric, introduced in \cite{Ebin1970a}, is also known as the Ebin metric. We call the metric \emph{natural} as it requires no additional background structure and is consequently invariant under the pushforward and pullback actions of the diffeomorphism group, i.e.,
\begin{equation}\label{eq:invariant}
G_g(h,k)=G_{\varphi^*g}(\varphi^*h,\varphi^*k)=G_{\varphi_*g}(\varphi_*h,\varphi_*k)
\end{equation}
for all $\varphi \in \Diff(M)$, $g \in \Met(M)$ and $h,k\in T_g\Met(M)$. 
Note that the invariance of the metric follows directly from the substitution formula for multi-dimensional integrals.

The Ebin metric induces a particularly simple geometry on the space $\Met(M)$, with explicit formulas for geodesics, geodesic distance and curvature. In the following theorem and corollary, we will present the most important of these formulas, which will be of importance for our proposed metric matching framework.

First we note that a metric $g\in\Met(M)$, in local coordinates, can be represented as a field of symmetric, positive-definite $n\times n$ matrices that vary smoothly over $M$. Similarly, each tangent vector at $g$ can be represented as a field of symmetric $n\times n$ matrices. By the results of \cite{freed1989basic,gil1991riemannian,clarke2013geodesics}, one can reduce the investigations of the space of all Riemannian metrics to the study of the geometry of the finite-dimensional space of symmetric, positive-definite $n \times n$ matrices. The point-wise nature of the Ebin metric allows one to solve the geodesic initial and boundary value problem on $\Met(M)$ for each $x\in M$ separately. Consequently the formulas for geodesics, geodesic distance and curvature on the finite-dimensional matrix space can be translated directly to results for the Ebin metric on the infinite-dimensional space of Riemannian metrics.

 Note that the space of Riemannian metrics, $\Met(M)$ with the Ebin metric, is not metrically complete and not geodesically convex. Thus the minimal geodesic between two Riemannian metrics may not exist in $\Met(M)$, but only in a larger space; the metric completion $\overline{
 \Met}(M)$, which consists of all possibly degenerate Riemannian metrics. This construction has been worked out in detail by~\cite{clarke2013completion} -- including the existence of minimizing paths in $\overline{\Met}(M)$. In the following proof, we will omit these details and refer the interested reader to the article~\cite{clarke2013completion}  for a more in-depth discussion. In Theorem~\ref{thm:ebingeodesics}, we present an explicit formula for the minimizing geodesic in $\overline{\Met}(M)$ that connects two given Riemannian metrics. 
\begin{theorem}[Minimizing geodesics]\label{thm:ebingeodesics}
	For $g_0, g_1\in\Met(M)$ we define
    \begin{align}
        k(x) &= \log\left(g_0^{-1}(x)g_1(x)\right),\quad k_0(x) = k(x) - \frac{\Tr(k(x))}{n}\Id\\
		a(x) &= \sqrt[4]{\det(g_0(x))},\quad b(x) = \sqrt[4]{\det(g_1(x))},\quad\ka(x) = \frac{\sqrt{n\Tr(k_0(x)^2)}}{4}\\
		q(t,x) &= 1+ t\left(\frac{b(x)\cos(\ka(x))-a(x)}{a(x)}\right),\quad	r(t,x) = \frac{t b(x)\sin(\ka(x))}{a(x)}.
	\end{align}
	Then the minimal path $g(t,x)$ with respect to the Ebin metric in $\overline{\Met}(M)$ that connects $g_0$ to $g_1$ is given by
	\begin{align}
	    g = \begin{cases}
	        \left(q^2+r^2\right)^{\frac2n}g_0\exp\left(\frac{\arctan(r/q)}{\ka}k_0\right) & 0<\ka<\pi,\\
	        q^{\frac4n}g_0 & \ka=0,\\
	        \left(1-\frac{a+b}{a}t\right)^{\frac4n}g_0\mathbbm{1}_{\left[0,\frac{a}{a+b}\right]}+\left(\frac{a+b}{b}t-\frac{a}{b}\right)^{\frac4n}g_1 \mathbbm{1}_{\left[\frac{a}{a+b},1\right]} & \ka\geq\pi,
	    \end{cases}
	\end{align}
where $\mathbbm{1}$ denotes the indicator function in the variable $t$. We suppressed the functions' dependence on $t$ and $x$ for better readability.
	
\end{theorem}

\begin{proof}
	This theorem is essentially a reformulation of the minimal geodesic formula given in \cite[Theorem 4.16]{clarke2013geodesics}. In fact, noting that the Ebin metric \eqref{eq.Ebinmetric} is point-wise, we can restrict ourselves to each $x\in M$. By \cite[Theorem 4.5]{clarke2013geodesics}, we know that in the case of $0\leq\ka(x)<\pi$, $g_1(x)$ is in the image of the Riemannian exponential map starting at $g_0(x)$, and the Riemannian exponential is a diffeomorphism between $U:=S^2T_x^*M\backslash (-\infty, -4/n]g_0(x)$ and its image. Here $S^2T_x^*M$ denotes the vector space of all symmetric (0,2) tensors at $x\in M$. Using the formula of the inverse exponential map in \cite[Theorem 4.5]{clarke2013geodesics} we calculate the preimage of $g_1(x)$,
	\begin{align}
		\on{Exp}_{g_0}^{-1}g_1\big|_x = \begin{cases}
		\frac4n\left(\frac{b(x)}{a(x)}\cos\kappa(x)-1\right)g_0(x)+\frac{b(x)\sin\kappa(x)}{\kappa(x) a(x)}g_0(x)k_0(x)  & 0<\ka(x)<\pi,\\[.5em]
		\frac4n\left(\frac{b(x)}{a(x)}-1\right)g_0(x) & \ka(x)=0.
		\end{cases}
	\end{align}
	The geodesic formula in the case of $0\leq\ka(x)<\pi$ then follows immediately from \cite[Theorem 4.4]{clarke2013geodesics}.
	For the case of $\ka(x)\geq\pi$, by \cite[Theorem 4.14]{clarke2013geodesics} the minimal geodesic between $g_0(x)$ and $g_1(x)$ is given by the concatenation of the straight segments from $g_0(x)$ to the zero tensor and from the zero tensor to $g_1(x)$, which gives the last statement and finally proves the result.
\end{proof}

An example of calculating a geodesic in the space $\overline{\Met}(M)$ using the explicit formula is visualized in Figure~\ref{fig:geo}.  The ellipse in the fifth row and fourth column, and the one in the first column and row of this figure are examples of geodesic paths that include possibly degenerate Riemannian metrics from the completion $\overline{\Met}(M)$. Note that handling these degenerate cases are well understood and we did not observe any numerical issues associated with such cases.

\begin{figure}[t!]
	\centering
	\includegraphics[width=\linewidth]{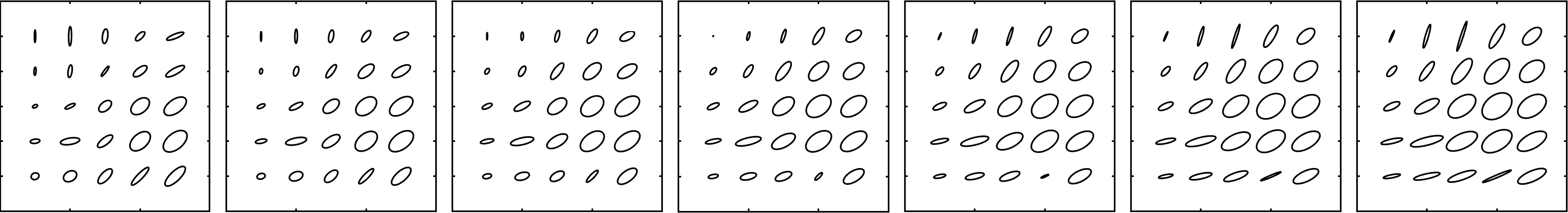}
	$t=0$\hspace{1.05cm} $t=1/6$\hspace{.87cm} $t=2/6$\hspace{.87cm} $t=3/6$\hspace{.87cm} $t=4/6$\hspace{.87cm} $t=5/6$\hspace{.87cm}
	$t=1$
	\caption{An example of an interpolating geodesic between two metric tensors on the grid with respect to the Ebin metric \eqref{eq.Ebinmetric}, where the left and the right ellipse fields represent the boundary metrics. One can observe the behavior of the Ebin metric by following the deformations of the ellipses representing the metric tensor. 
	}
	\label{fig:geo}
\end{figure}
Next, we recall that the geodesic distance of a Riemannian metric is defined as the infimum of all paths connecting two given points,
\begin{equation}\label{eq:geodisdef}
\dist_{\Met}(g_0,g_1)=\inf \int_0^1 \sqrt{G_g(\partial_t g,\partial_t g)} dt,
\end{equation} 
where the infimum is taken over all paths $g:[0,1]\to \Met(M)$ with 
$g(0)=g_0$ and $g(1)=g_1$. As a direct consequence of Theorem~\ref{thm:ebingeodesics}, we obtain the explicit formula for the distance function given in Corollary~\ref{cor:geodesicdist}.
\begin{corollary}[Geodesic distance]\label{cor:geodesicdist}
	Let $g_0, g_1\in\Met(M)$ and let $k, k_0$, $a$, $b$ and $\ka$ be as in Theorem~\ref{thm:ebingeodesics}. Let 
$
	\theta(x) = \min\left\{\pi, \ka(x)\right\}.
$
	Then the squared geodesic distance of the Ebin metric is given by
	\begin{align}\label{eq:distance_function}
	\dist_{\Met}(g_0, g_1)^2 = \frac{16}{n}\int_M \left(a(x)^2 - 2a(x)b(x)\cos\left(\theta(x)\right) + b(x)^2\right)dx.
	\end{align}
\end{corollary}
\begin{proof}
	Using Theorem~\ref{thm:ebingeodesics}, we obtain the formula of the minimal geodesic that connects $g_0$ and $g_1$. By calculating the time derivative $\partial_tg(t)$, the final statement follows from the definition of the geodesic distance \eqref{eq:geodisdef} by a direct computation.
\end{proof}

Having equipped the space of Riemannian metrics with the distance function~\eqref{eq:distance_function}, we can consider the Fr\'echet mean, $\hat g$, of a collection of metrics, $g_1,\ldots g_N$, which is defined as a minimizer of the sum of squared distances.
\begin{equation}
\hat{g}=\underset{g }{\operatorname{\rm argmin}}
    \sum_{i=1}^{N}\operatorname{dist}_{\Met}^{2}(g,g_i). 
\end{equation}
One could directly minimize this functional using a gradient-based optimization procedure. As our distance function is the geodesic distance function of a Riemannian metric and since we have access to an explicit formula for the minimizing geodesics, we will instead use the iterative geodesic marching algorithm, see e.g.~\cite{ho2013recursive}, to approximate the Fr\'echet mean: Given $N$ Riemannian metrics $g_i$, we approximate the Fr\'echet mean via $\hat g=\hat g_{N}$, where $\hat g_{i}$ is recursively defined as 
 $\hat g_0=g_0$, $\hat g_i(x)= g(1/(i+1),x)$ and where
 $g(t,x)$ is the minimal path, as given in Theorem~\ref{thm:ebingeodesics}, connecting $\hat g_{i-1}$ to the $i$-th data point $g_i$.
Thus one only has to calculate $N$ geodesics \emph{in total} in the space of Riemannian metrics, whereas a gradient-based algorithm would require one to calculate $N$ geodesic distances \emph{in each step} of the gradient descent.

\subsection{The Induced Distance Function on the Diffeomorphism Group}\label{section:diffmetric}
We can use the geodesic distance function of the Ebin metric to induce a right-invariant distance function on the group of diffeomorphisms. As we will be using this distance function as a regularization term in our matching functional, we will briefly describe this construction here. We fix a Riemannian metric $g\in \Met(M)$ and define the ``distance'' of a diffeomorphism $\varphi$ to the identity via
\begin{equation}
    \operatorname{dist}_{\Diff}^{2}(\operatorname{id},\varphi) = \operatorname{dist}_{\Met}^{2}(g,\varphi^*g)=\operatorname{dist}_{\Met}^{2}(g,\varphi_*g),
\end{equation}
where the last equality is due to the invariance of the Ebin metric. 
To be more precise, this distance can be degenerate on the full diffeomorphism group since the isometries of the Riemannian metric $g$ form the kernel of $\operatorname{dist}_{\Diff}$. For our purposes we will consider the Euclidean metric for the definition of $\operatorname{dist}_{\Diff}$. Thus the only elements in the kernel are translations and rotations. The right invariance of $\operatorname{dist}_{\Diff}$ follows directly from the $\Diff(M)$-invariance of the Ebin metric. We note, however, that $\dist_{\Diff}$ is not directly associated with a Riemanian structure on the diffeomorphism group: the orbits of the diffeomorphism group in the space of metrics are not totally geodesic and thus $\dist_{\Diff}$ is not the geodesic distance of the pullback of the Ebin metric to the space of diffeomorphisms. See also~\cite{KLMP2013} where this construction has been studied in more detail.

\section{Computational Anatomy of the Human Connectome}
Fundamental to the precise characterization and comparison of the human connectome of an individual subject or a population as a whole is the ability to map or register two different human connectomes. The framework of Large Deformation Diffeomorphic Metric Mapping (LDDMM) is well developed for registering points~\citep{joshi2000landmark}, curves~\citep{glaunes2008large} and surfaces~\citep{vaillant2005surface} all modeled as submanifolds of $\R^{3}$ as well as images modeled as an $L^2$ function~\citep{beg2005computing}.
This framework has also been extended to densities~\citep{bauer2015diffeomorphic} modeled as volume forms. We now extend the diffeomorphic mapping framework to the  connectome modeled as Riemannian metrics. The diffeomorphism group deforms the domain, $M$, and acts naturally on the space of metrics, $\Met(M)$, see Equation~\eqref{eq:diff_action}. With this action and a reparameterization-invariant metric, the problem of registering two connectomes  fits naturally into the framework of computational anatomy. 
 
Now the registration of two connectomes is achieved by minimizing the  energy
\begin{equation}
E(\varphi)= \operatorname{dist}_{\Diff}^{2}(\operatorname{id},\varphi)+\lambda_1 \operatorname{dist}_{\Met}^{2}(g_0,\varphi_*g_1)
\label{eq:energy-metric}
\end{equation}
over all such diffeomorphisms in $\Diff(M)$. 
Here  $\operatorname{dist}_{\Diff}$ is a right invariant distance on $\Diff$ and $\operatorname{dist}_{\Met}$ is a reparameterization-invariant distance  on the space of all Riemannian metrics, e.g., the geodesic distance of the metrics studied above.  The first term measures the deformation cost and the second term is a similarity measure between the target and the deformed source connectome.  The invariance of the two distances is essential for the minimization problem to be independent of the choice of coordinate system on the brain manifold. 

We use the distance function as introduced in Section~\ref{section:diffmetric} to measure the deformation cost, i.e., $\operatorname{dist}_{\Diff}(\operatorname{id},\varphi)=\operatorname{dist}_{\Met}(g,\varphi_*g)=\operatorname{dist}_{\Met}(g,\varphi^*g)$ where $g$ is the restriction of the Euclidean metric to the brain domain. 
This choice greatly increases computational efficiency since we can now use the formulas from Section~\ref{sec:riemet} as explicit formulas for both terms of the energy functional.
To minimize the energy functional, we use a gradient flow approach described in Algorithm~\ref{algo1}, where the gradient on $\Diff(M)$ is calculated with respect to a right invariant Sobolev metric of order one, which  at the identity  is given by
\begin{equation}
G_{\operatorname{id}}(u,v)=\int g(\Delta u,v)dx+\lambda \sum_{i=1}^k g(u,\xi_i) g(v,\xi_i) dx
\end{equation}
where $\lambda>0$ is a weight parameter, $g$ is the restriction of the Euclidean metric to the brain domain and  $\xi_1,\ldots \xi_k$ is an orthonormal basis of the harmonic vector fields on $M$.  This metric, sometimes  
called the information metric, has been first introduced by \cite{Modin2015}; see also~\cite{bauer2015diffeomorphic}. We choose this specific gradient because of the relation of the information metric to both the Ebin metric on the space of metrics and the Fisher-Rao metric on the space of probability densities. We summarize the relations between these geometries in~Figure~\ref{fig:DiffMetCommute};  for a precise description of the underlying geometric picture we refer to~\cite{KLMP2013} and~\cite{bauer2015diffeomorphic}. 
\begin{figure}
    \centering
    \begin{tikzcd}[ampersand replacement=\&, column sep=small, row sep=huge]
 \left(\Diff(M) \bign/ \operatorname{Iso}(\overline{g}),G^I\right) \arrow[rr, "\rm II : \varphi \mapsto \varphi^*\overline{g}"] \arrow[dr, "\rm I : \varphi \mapsto \varphi^* \mu_0"'] \& \&  \left(\Met(M),G^E\right) \arrow[dl, "\rm III : g \mapsto \vol(g)"] \\
\&\left(\operatorname{Dens}(M),G^{FR}\right)  \& 
    \end{tikzcd}
    \caption{Relations between the information metric on the diffeomorphism group, the Ebin metric on the space of Riemannian metrics, and the Fisher-Rao metric on the space of densities. The mappings I and III are Riemannian submersions and the mapping II is an isometric embedding. Furthermore the diagram is commutative, i.e., $\operatorname{I}=\operatorname{II}\circ \operatorname{III}$. Note, that the Ebin metric $G^E$ and the Fisher-Rao metric $G^{FR}$ are of order zero, while the information metric $G^I$ is a first order Sobolev metric. This discrepancy in the orders of the metric is explained by the fact that the mappings $\operatorname{I}$ and $\operatorname{II}$ contain a derivative.}
    \label{fig:DiffMetCommute}
\end{figure}

Note, that our framework allows for the immediate inclusion of  points, curves, surfaces and images in the registration problem. Image intensity information, for example, can be easily incorporated in the registration problem by simply adding an appropriate similarity measure for the image term (e.g. the standard $L^2$ metric between the deformed moving image and the fixed image) to the energy functional. The minimization problem incorporating image intensities naturally becomes:
\begin{equation}
E(\varphi)= \operatorname{dist}_{\Diff}^{2}(\operatorname{id},\varphi)+\lambda_1 \operatorname{dist}_{\Met}^{2}(g_0,\varphi_*g_1) +\lambda_2 \operatorname{dist}_{L^2}^{2}(I_0,\varphi_*I_1) \ ,
\label{eq:energy-joint}
\end{equation}
where $\lambda_1,\lambda_2$ are the relative weights and $\varphi_*I_1 = I_1\circ \varphi^{-1} $ is the natural left action of Diff on $L^2$ images.  In this convention $I_1$ is called the moving image ($g_1$ the moving metric, resp.). It is important to note that the energy functional only depends on $\varphi^{-1}$ and not on $\varphi$. To see this we use the invariance of the geodesic distance on $\Diff$ and the definition of the pushforward to rewrite~\eqref{eq:energy-joint} as
\begin{equation}
E(\varphi)= \operatorname{dist}_{\Diff}^{2}(\operatorname{id},\varphi^{-1})+\lambda_1 \operatorname{dist}_{\Met}^{2}\left(g_0,g_1(T\varphi^{-1}\cdot,T\varphi^{-1}\cdot)\right) +\lambda_2 \operatorname{dist}_{L^2}^{2}(I_0,I_1\circ\varphi^{-1}).
\label{eq:energy-joint2}
\end{equation}
Consequently, we will use $\varphi^{-1}$ rather than $\varphi$ as our optimization variable, c.f. Algorithm~1 below.

\algnewcommand{\IIf}[1]{\State\algorithmicif\ #1\ \algorithmicthen}
\algnewcommand{\EndIIf}{\unskip\ \algorithmicend\ \algorithmicif}
\begin{algorithm}[h]
\caption{Inexact Metric Matching Algorithm}\label{algo1} 
    \begin{algorithmic}
        \Inputs{Fixed and moving metrics $g_0$, $g_1$; Fixed and moving images $I_0$, $I_1$}
        \Initialize{step size $\epsilon$; weight parameters $\lambda_1,\lambda_2$; max iteration times $\operatorname{MaxIter}$} 
        \State{$\varphi^{-1},E\leftarrow\operatorname{id},0$}
       
        \For{$\operatorname{iteration}=0:\operatorname{MaxIter}$}
            \State{$\varphi_*g_1\leftarrow(d\varphi^{-1})^T(g_1\circ\varphi^{-1})(d\varphi^{-1})$}\Comment{Pushforward of $g_1$ by $\varphi$}
            \State{$\varphi_*I_1\leftarrow I_1\circ\varphi^{-1}$}\Comment{Pushforward of $I_1$ by $\varphi$}
            
            \State{$E\leftarrow \operatorname{EbinEnergy}(\varphi_*g_1,g_0,\varphi_*I_1,I_0,\lambda_1,\lambda_2)$}\Comment{Calculate energy by Equation \eqref{eq:energy-joint}}
            
            \State{$v\leftarrow\Delta^{-1}(\operatorname{E.grad})$}\Comment{Transfer gradient w.r.t. information metric to $L^2$}
            \State{$\psi\leftarrow\operatorname{id}-\epsilon v$}\Comment{Construct the approximation} 
            \State{$\varphi^{-1}\leftarrow\psi\circ\varphi^{-1}$}\Comment{Update the diffeomorphism}
        \EndFor
        \State\Return{$\varphi^{-1}$}
    \end{algorithmic}  
\end{algorithm}

\subsection{Estimating the Atlas for a Population of Connectomes}\label{sec:atlas}
Given a collection of connectomes modeled as points on an abstract Riemannian manifold, we can directly apply least-squares estimation to define the average connectome. Thus the template estimation problem can be formulated as a joint minimization problem:
\begin{align}
   \hat{g} = \underset{g,\varphi_i }{\operatorname{\rm argmin}}
    \sum_{i=1}^{N}\operatorname{dist}_{\Diff}^{2}(\operatorname{id},\varphi_i)+ \lambda_1 \operatorname{dist}_{\Met}^{2}(g,{(\varphi_i)}_*g_i)
    \label{eq:atlasmetric}
\end{align}
We use the iterative alternating algorithm proposed in \cite{joshi2004unbiased} for solving the above optimization problem: we alternate gradient steps between optimizing with respect to each diffeomorphism, $ \varphi_i , i=1,\cdots,N $, and minimizing with respect to the metric average $\hat g$. In the metric optimization step we use the  Fr\'echet mean algorithm described in Section~\ref{sec:riemet}.  The above procedure for estimating a connectome atlas can again be trivially extended to jointly estimate an image atlas consistent with the connectome atlas. To estimate the atlases jointly we use the following extended  
joint minimization problem:
\begin{align}
   \hat{g},\hat{I} = \underset{g,I,\varphi_i }{\operatorname{\rm argmin}}
    \sum_{i=1}^{N}\operatorname{dist}_{\Diff}^{2}(\operatorname{id},\varphi_i)+ \lambda_1 \operatorname{dist}_{\Met}^{2}(g,(\varphi_i)_*g_i) + \lambda_2 \operatorname{dist}_{L^2}^{2}(I,(\varphi_i)_*I_i)\ \ . \label{eq:atlasjoint}
\end{align}
For the extended image-connectome atlas alternating algorithm at each iteration of the algorithm, the atlas image is updated by the simple average of the deformed individual subject images. This is a consequence of using the simple $L^2$ metric for the images.  See Algorithm~\ref{algo2} for details of this process.

\begin{algorithm}[h]
\caption{Atlas Building Algorithm}\label{algo2}
    \begin{algorithmic}
        \Inputs{metric fields $G = \{g_1,\dots, g_N\}$, images $I = \{I_1,\dots, I_N\}$}
        \Initialize{max iteration times $\operatorname{MaxIter}$}
        
        \For{$\operatorname{iteration}=0:\operatorname{MaxIter}$}
            \State $g_{\operatorname{mean}}\leftarrow\operatorname{FrechetMean}(G)$\Comment{Section~\ref{sec:riemet}, compute with random order each time}
            \State $I_{\operatorname{mean}}\leftarrow\operatorname{EuclideanMean}(I)$
            \For{$i=1:N$}
                \State{$\varphi^{-1}_i\leftarrow\operatorname{MetricMatching}(g_{\operatorname{mean}},g_i,I_{\operatorname{mean}},I_i)$}\Comment{Algorithm~\ref{algo1}}
                \State{$g_i\leftarrow(\varphi_i)_*g_i$}\Comment{Update $g_i$ by pushforward of $\varphi_i$}
                \State{$I_i\leftarrow(\varphi_i)_*I_i$}\Comment{Update $I_i$ by pushforward of $\varphi_i$}
            \EndFor
        \EndFor
        \State\Return{$g_{\operatorname{mean}}$}
    \end{algorithmic}  
\end{algorithm}  

\subsection{Implementation Details}
In this section, we describe some implementation details that are important for reliable and efficient algorithm performance, particularly when applied to 3D brain data.
\subsubsection{Metric Estimation} As done in \cite{hao2014improved}, we apply a mask to both the connectome metric estimation process and the atlas-building algorithm for two reasons.  First, it is important that we constrain the problem to biologically realistic white matter tracts by not allowing tractography to flow through regions of CSF.  Second, we avoid numeric issues associated with processing air and other noisy regions outside the skull.  This also speeds up computation, as we need to look only at voxels inside the masked region instead of the entire image volume. Care must be taken with both first and second derivatives to use an appropriate and accurate finite difference stencil near the boundaries of the mask to ensure only points inside the mask are used.  When matching the geodesics to a vector field consisting of the principal eigenvector directions, it is important to ensure that the eigenvector signs are consistent prior to computing derivatives of the vector field. 

\subsubsection{Atlas Building} For the atlas building algorithm, we deform each individual mask into atlas space at each outer iteration, and then apply the union of these deformed masks when computing the current atlas estimate. For each iteration of the atlas building algorithm, we perform only two iterations inside the metric matching function to avoid overfitting the individual metrics to early estimates of the Fr\'echet mean. Though we have not observed any subject order-related problems with the Fr\'echet mean calculation in practice, we randomize the order of the subjects each time we compute the mean to avoid the possibility of introducing problems related to subject order. In practice, we find the algorithm behaves well when we update $\epsilon$ in Algorithm~\ref{algo1} such that $1/\epsilon$ is approximately equal to the energy from \eqref{eq:energy-joint}.

For computational efficiency, all our algorithms are implemented in \texttt{PyTorch}, which allows us to take advantage of the built-in GPU acceleration and automatic differentiation.
The computationally most expensive part in the  geodesic distance of the Ebin metric is the calculation of the matrix logarithm, i.e., the calculation of $k = \log(g_0^{-1}g_1)$. Currently the matrix logarithm function is not implemented in \texttt{PyTorch}, and other alternatives such as the function provided by \texttt{Scipy} do not support automatic differentiation. Therefore we  calculate the Cholesky factorization of $g_0=GG^T$. 
 We aim to use this factorization to reduce the eigendecomposition of the nonsymmetric matrix $g_0^{-1}g_1$ to the eigendecomposition of the symmetric matrix  $W=G^{-1}g_1(G^T)^{-1}$: 
 writing 
 $W = Q\Lambda Q^T$ for the eigendecomposition of $W$, we directly obtain the eigendecomposition of $g_0^{-1}g_1= V\Lambda V^{-1}$,
where $V = (G^T)^{-1}Q$.
This in turn allows us
to calculate $k = \log(g_0^{-1}g_1) = V\log(\Lambda)V^{-1}$.

Consequently, the bottleneck of our atlas building algorithm is the large amount of eigendecomposition problems that have to be computed -- in each iteration step we have to solve $N \times \operatorname{Res}$ eigendecompositions of $n\times n$ matrices, where $\operatorname{Res}$ is the resolution of the image and $n$ is the dimension of the domain. 

It turns out that the GPU implementation of \texttt{torch.linalg.eig} is not well-designed for solving eigendecomposition of numerous small matrices. The extremely low speed made our algorithms unusable for any experiments in 3D. To speed up these calculations, we re-implement the eigendecomposition based on the work of \cite{lenssen2019}, which leads to an order of magnitude increase in performance. 

Powered by the \texttt{torch.autograd} module, we can now easily solve the gradient of the Ebin energy in Algorithm~\ref{algo1} without a closed-form gradient solution. Our 3D atlas building code, including both metric-only matching and joint matching, will become publicly available at \url{https://github.com/aarentai/Atlas-Building-3D}.

\subsubsection{Geodesic Tractography}
Several computational strategies have been proposed to compute white matter tractography as geodesic curves. \citet{o2002new} develop a level set approach to solve the Eikonal-type equation of the geodesic distance transform from a seed point. \citet{fletcher2007} extend this approach to simultaneously solve for the entire set of minimal geodesics between two brain regions by solving two Hamilton-Jacobi PDEs. These Hamilton-Jacobi PDEs can be solved quickly on GPUs using the method of \citet{jeong2007}.  While these strategies are better choices for a production pipeline, for this paper we chose to directly integrate the geodesic equation from seed points using the principal eigenvector at each seed as the initial direction for shooting.

Given a Riemannian metric $g$, we compute the corresponding Christoffel symbol $\Gamma$ via
\begin{equation}
    \Gamma_{ij}^k=\frac{1}{2}\sum_{l}^n  g^{kl}\left(\frac{\partial  g_{jl}}{\partial x^i}+\frac{\partial  g_{il}}{\partial x^j}-\frac{\partial  g_{ij}}{\partial x^l}\right),
\end{equation}
where $g^{ij}$ denotes the components of the inverse metric tensor. Together with the position $\gamma$ and velocity $\Dot{\gamma}$, the Christoffel symbols enable us to find the acceleration $\Ddot{\gamma}$ at time $t$  by solving the geodesic equation $\Ddot{\gamma}^{(k)}+\Gamma_{ij}^k\Dot{\gamma}^{(i)}\Dot{\gamma}^{(j)}=0$, which gives
\begin{equation}
\Ddot{\gamma}^{(k)}(t)=-\left[\Dot{\gamma}^{(1)}(t)\; \Dot{\gamma}^{(2)}(t)\;  \Dot{\gamma}^{(3)}(t)\right]
    \left[\begin{array}{ccc}
        \Gamma^k_{11}(\gamma(t)) & \Gamma^k_{12}(\gamma(t)) & \Gamma^k_{13}(\gamma(t)) \\
        \Gamma^k_{21}(\gamma(t)) & \Gamma^k_{22}(\gamma(t)) & \Gamma^k_{23}(\gamma(t)) \\
        \Gamma^k_{31}(\gamma(t)) & \Gamma^k_{32}(\gamma(t)) & \Gamma^k_{33}(\gamma(t)) \\
    \end{array}\right]
    \left[\begin{array}{c}
         \Dot{\gamma}^{(1)}(t) \\
         \Dot{\gamma}^{(2)}(t) \\
         \Dot{\gamma}^{(3)}(t)
    \end{array}\right]
\end{equation}
where $\Ddot{\gamma}^{(k)},\Dot{\gamma}^{(k)}$ are the components of the acceleration vector $\Ddot{\gamma}$ and velocity vector $\Dot{\gamma}$. 

After the initial conditions for the position and velocity are given and $\Ddot{\gamma}^{(k)}(0)$ is computed, we update the acceleration at subsequent time steps using a fourth-order Runge-Kutta scheme~\citep{press1986numerical}. See also the function \texttt{algo.geodesic.geodesicpath()} in our open access repository. 

\section{Results}
We first demonstrate our framework in 2D using synthetic data and 2D data extracted from brain images.  Next, we construct a 3D connectome atlas from DWMRI for a subset of subjects from the Human Connectome Project.  We show that we can use the complementary information from T1-weighted MRI for those same subjects to build an integrated multimodal atlas.  Finally, we demonstrate that the multimodal atlas preserves both local and long-range connectivity information by computing both whole-brain and seed region-based geodesic tractography of the atlas.

\subsection{2D Simulated Data}
We verified our method by generating vector fields whose central integral curves are a family of parameterized cubic functions. We used the method of parallel curves to add vectors for additional integral curves parallel to the central curve with a distance $k \in [-0.2, 0.2]$ from the central curve.  We then constructed tensors whose principal eigenvectors align with the generated vector fields and that have a specified major axis to minor axis ratio of 6:1.  

We first estimated the adaptive metric conformal to the inverse-tensor metric such that the geodesics of the adaptive metrics align with the integral curves of the simulated vector fields. After finding the connectome metric for each subject, we  ran 400 iterations of the atlas building Algorithm~\ref{algo2} using only the metric distance as shown in~\eqref{eq:atlasmetric} to estimate the atlas in Figure~\ref{fig:cubicmeans}. To help the diffeomorphisms update smoothly, we set $\lambda_1 = 100$ in~\eqref{eq:energy-metric} and the step size $\epsilon=5$ in Algorithm~\ref{algo1}. 

We compared a geodesic of the atlas starting from a particular seed point with geodesics of the four connectome metrics starting from the atlas seed point mapped into individual space. Figure~\ref{fig:cubicmeans} shows these individual geodesics in atlas space before and after applying the diffeomorphisms.  We see that the atlas geodesic is nicely centered in the middle of the undeformed individual geodesics as expected. Also, the deformed individual geodesics align well with the atlas geodesic. The upper right panel shows the mean distance between the atlas geodesic and deformed subject geodesic in atlas space in different stages. After 75 iterations of optimization, we can see the distance of the four pairs of geodesics almost converged.

\begin{figure}[h]
\centering
\includegraphics[width=\linewidth]{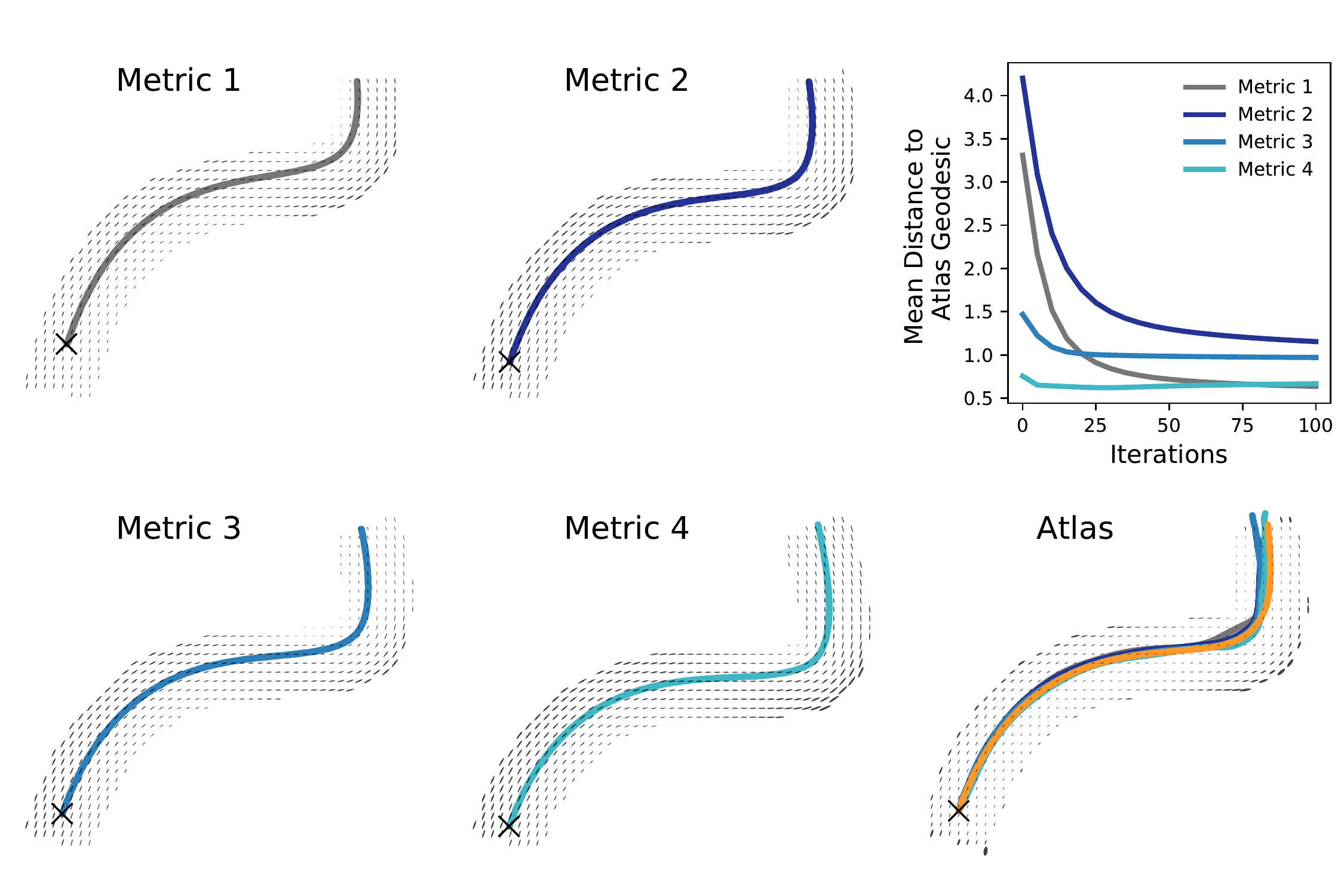}
\caption{Left and center: geodesics of four synthetic metrics starting from the atlas seed point (X) mapped into each metric's space. Upper right: mean distance trend between atlas geodesic and every deformed subject geodesic in atlas space. Lower right: estimated atlas with geodesic (orange) overlaid on geodesics from the four metrics deformed into atlas space. }
\label{fig:cubicmeans}
\end{figure}
\subsection{2D Real Data}
To illustrate how our method works with real data, we used two subjects from the Human Connectome Project Young Adult (HCP) dataset~\cite{glasser2013minimal}.  For each subject, we fit a diffusion tensor model to the images with a $b$-value of 1000 using \verb!dtifit! from FSL~\cite{basser1994estimation} and generated a white-matter mask based on fractional anisotropy values thresholded between 0.25 and 1. To process in 2D, we extracted a single axial slice from each image along with the x and y components of the associated tensors. We estimated the adaptive connectome metric from the inverse-tensor metric associated with the 2D diffusion tensors. We would like to emphasize that the 2D atlas construction is only for illustrative purposes. In real human data no tract is expected to live in a single plane and 2D processing is not appropriate. 

To generate the atlas shown in Figure~\ref{fig:ebinbraintoatlas}, we ran atlas building with only the metric distance terms from~\eqref{eq:atlasmetric} for 5000 iterations with $\lambda_1=100$, $\epsilon=1$, which took 50 minutes on an Intel Xeon Silver 4108 CPU. The regularization term, $\lambda_1$, balances the magnitudes of the diffeomorphisms from each subject's connectome metric to the atlas. To ensure that the final geodesics in the atlas also follow the major eigenvectors of the atlas tensors, we solve for the $\alpha$ conformal factor for the atlas as described in Section~\ref{sec:structasmet}. 

\begin{figure}[h]
\centering
\includegraphics[width=\linewidth]{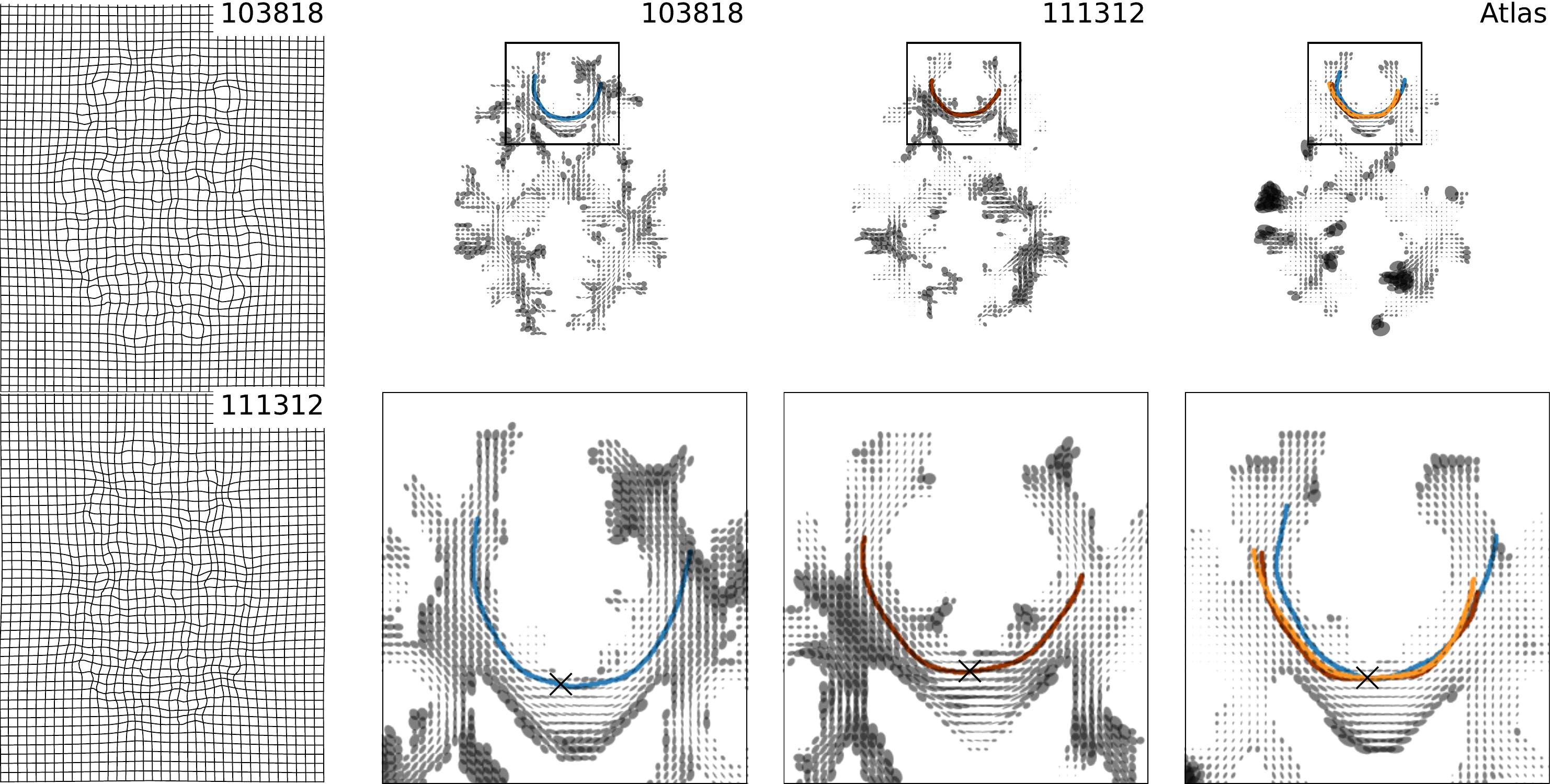}
\caption{Left: diffeomorphism from HCP subjects (103818, 111312) to the atlas. Center: each subject's connectome metric and a geodesic (blue, red) starting from the atlas seed (X) mapped to subject space.  Right: atlas and a geodesic (orange) starting at the seed (X). Subject geodesics are mapped to atlas space (blue, red). Bottom: detailed view of corpus callosum.}
\label{fig:ebinbraintoatlas}
\end{figure}

\subsection{3D Metric Atlas}
We used six subjects from the Human Connectome Project Young Adult (HCP) dataset~\cite{glasser2013minimal} in this experiment as shown in Figure~\ref{fig:HCPsubjects}.  For each subject, we fit a diffusion tensor model to the images using $b$-values of 1000, 2000, and 3000 using \verb!dtifit! from FSL~\cite{basser1994estimation} and generated a white-matter mask by keeping voxels with fractional anisotropy between 0.25 and 1. We rigidly registered the T1 images for each subject using \verb!brainsfit! from~\citet{johnson2007brainsfit}, and applied that rigid registration to the subject's white matter mask and diffusion tensors, being sure to reorient the individual tensors.  We estimated $\alpha$ for the adaptive connectome metric from the inverse-tensor metric associated with the diffusion tensors after smoothing them with a Gaussian filter, $\sigma=1.5$ and cleaning any tensors that were not positive semi-definite. The estimated $\alpha$ was clipped to the range $[-2, 2]$ before applying it the inverse-tensor metric associated with the unsmoothed diffusion tensors to create the adaptive connectome metric.  In the atlas building process, we used the metric distance terms from~\eqref{eq:atlasmetric} and set the step size $\epsilon = 0.005$ in Algorithm~\ref{algo1} and $\lambda_1 = 1.0$. For 800 iterations, the algorithm took about 12 hours on an Nvidia Titan RTX GPU, which features 24GB VRAM. 

To generate the T1 atlas for this experiment, we computed the mean of each subject's T1 image after deforming it by the diffeomorphism found from the metric atlas construction.  As shown in Figure~\ref{fig:metvsjoint}, the white matter atlas looks reasonable, but the T1 atlas is blurry, especially in gray matter regions because the atlas construction process is not using any gray matter information to guide the diffeomorphisms in these regions.

\begin{figure}[h]
\centering
\includegraphics[width=1\linewidth]{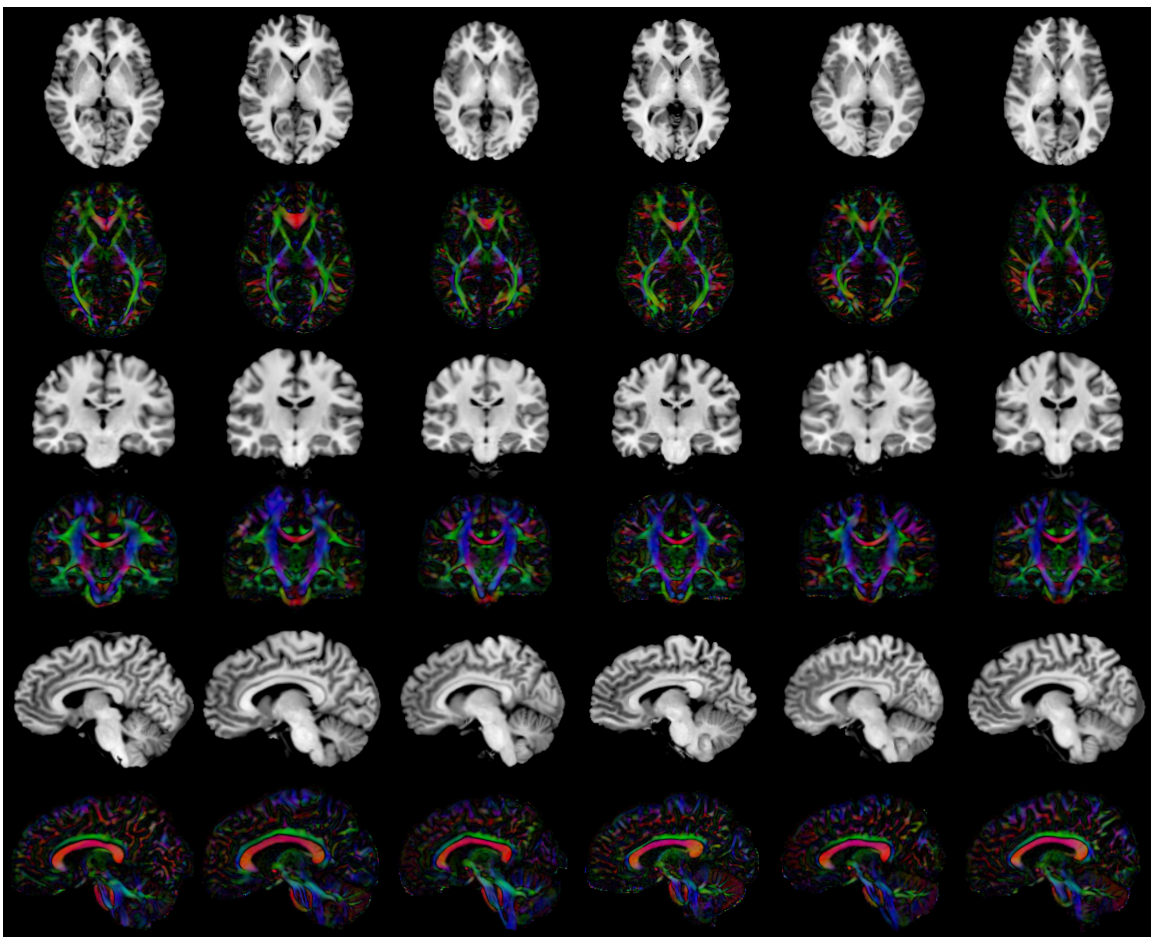}
\caption{Axial, coronal, and sagittal views of the 6 HCP subjects' T1 images and their metrics colored by the direction of the principal eigenvector of the inverse of the metric.  Red is oriented from left to right, green from anterior to posterior and blue from inferior to superior directions.}
\label{fig:HCPsubjects}
\end{figure}

\begin{figure}[h]
\centering
\def\svgwidth{\textwidth}
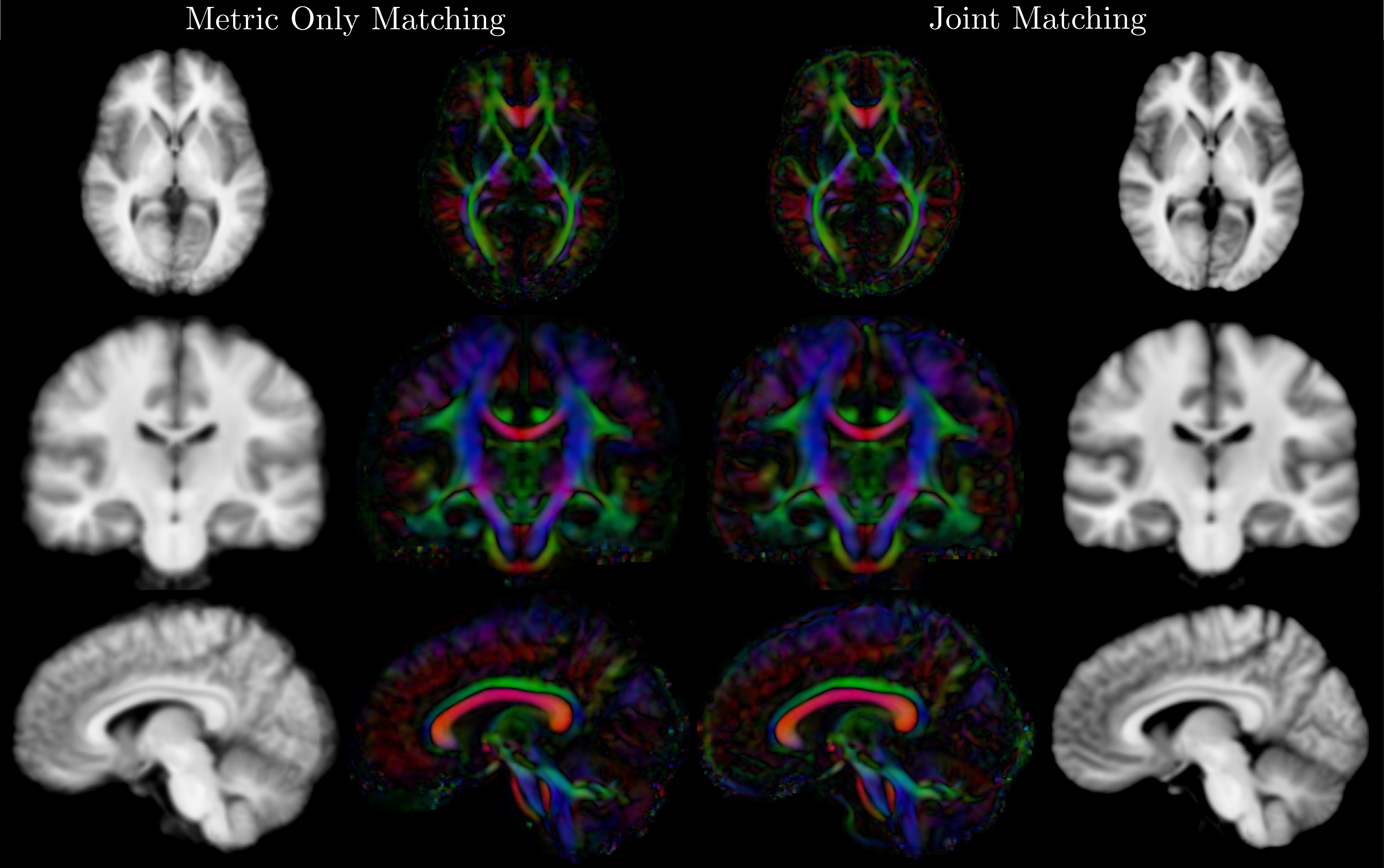
\caption{Axial, coronal and sagittal views of the T1 and metric atlas.  Left columns: the atlas produced using only the metric matching term.  Right columns: the atlas produced using metric and image matching terms jointly.}
\label{fig:metvsjoint}
\end{figure}

\subsection{3D Joint Atlas}
We used the same six subjects to construct a joint T1 and metric atlas by adding in the image matching term as shown in~\eqref{eq:atlasjoint}.  We set the weight of the image term, $\lambda_2$ in~\eqref{eq:energy-joint}, to $1.0 \times 10^{-8}$, which balances the image term loss to approximately the same magnitude as the metric term loss. All other parameters were kept the same as in the metric-only atlas building process.  The resulting metric and T1 atlases are shown in Figure~\ref{fig:metvsjoint}.  Note how much more well-defined the gray matter regions are when the T1 information is included in the optimization process.  At the same time, the metric atlas is not degraded by adding this term in, and may even be slightly better than before.  
\begin{figure}[h]
\centering
\includegraphics[width=\linewidth]{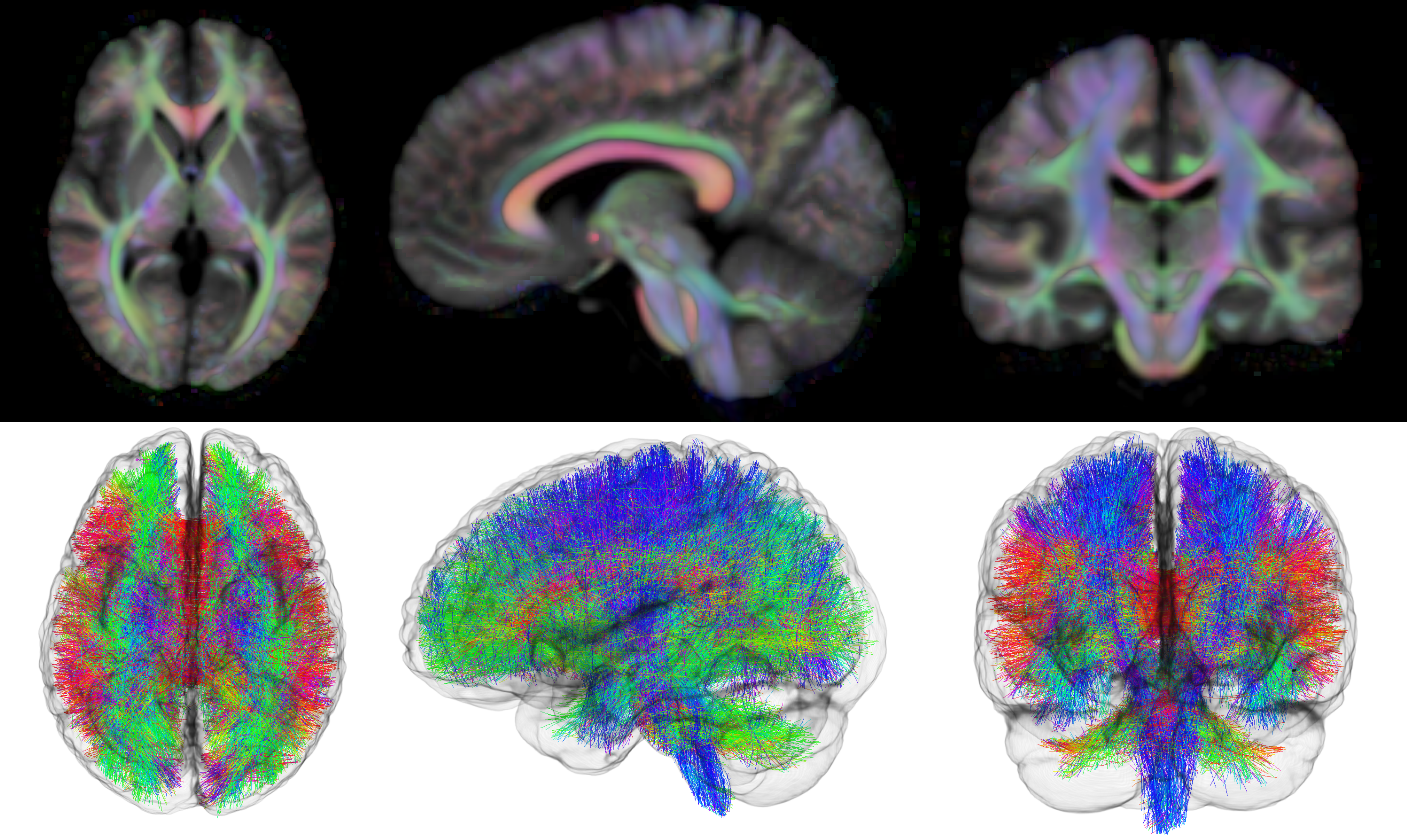}
\caption{Top row: metric colored by orientation overlaid on T1 for joint atlas of six HCP individuals.  Bottom row: whole-brain geodesic tractography for joint atlas, with tract segments colored by their orientation.}
\label{fig:jointatlas}
\end{figure}

The whole-brain tractography shown in Figure~\ref{fig:jointatlas} was performed by seeding eight geodesics per voxel in the white matter mask, stopping when each geodesic left the averaged white matter mask.  Individual tracts were created by drawing seed regions in the atlas space and transforming the regions back to subject space. Each voxel in the seed region had 27 seeds.  Geodesics terminated when fractional anisotropy dropped below 0.2 for the subject tractography, or when the geodesic left the averaged white matter mask for the atlas tractography.  Other than the seed region and FA thresholding, we did not apply any stopping or rejection criteria based on angle, tract length, anatomical priors or any of the other techniques used to clean up false positives.  We performed tractography visualization using 3D Slicer (www.slicer.org) via the SlicerDMRI project (http://dmri.slicer.org)~\citep{norton2017slicerdmri,zhang2020slicerdmri}.
Note that the whole-brain tractography can also be registered to the O'Donnell atlas using the \texttt{WhiteMatterAnalysis} toolkit~\citep{o2007automatic,o2012unbiased,zhang2018anatomically} in order to do clustering based on that atlas if desired.  Example tracts for the genu of the corpus callosum, corticospinal tract, and cingulum computed by integrating geodesics of the joint atlas are shown in Figure~\ref{fig:3datlastracts}.

\begin{figure}[h]
\centering
\includegraphics[width=.9\linewidth]{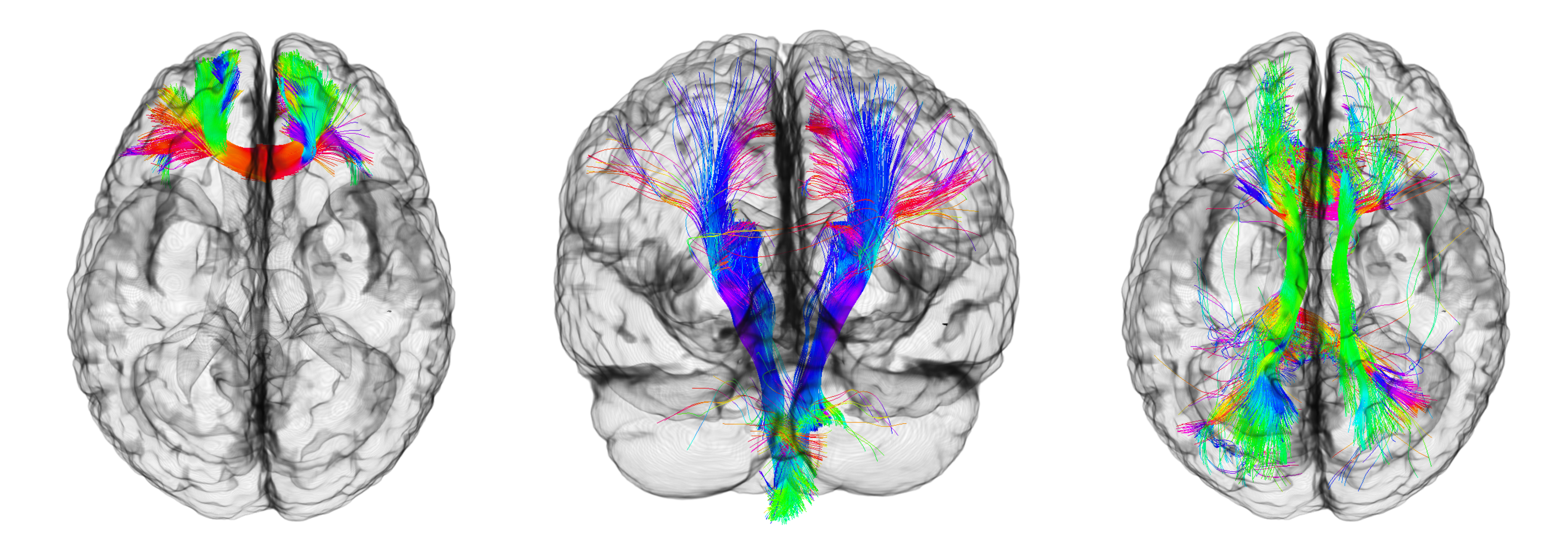}
\caption{Examples of joint atlas tractography generated from seeds placed in atlas seed regions, with tract segments colored by orientation. Left: axial view of tractography seeded in the genu of the corpus callosum. Center: coronal view of tractography seeded in the corticospinal tract.  Right: axial view of tractography seeded in the cingulum.}
\label{fig:3datlastracts}
\end{figure}

\section{Conclusions}
In this paper, we introduce a novel framework for statistically analyzing structural connectomes by representing them as a point on the manifold of Riemannian metrics, enabling us to perform geometric statistics. Using this representation, we build a framework for connectome atlas construction based on the action of the diffeomorphism group and the natural Ebin metric on the space of all Riemannian metrics. Because this framework is compatible with existing image atlas construction frameworks, we are then able to construct an integrated multimodal atlas using complementary white matter and cortical information from DWMRI and T1-weighted MRI simultaneously.  
In future work we plan to investigate in more detail the convergence properties of the proposed algorithms and quantitatively compare  our approach to other existing methods. We expect this new methodology to open up opportunities for a deeper understanding of structural connectomes, their variabilities and their relationships to cortical structure.

\acks{M. Bauer was supported by NSF grants DMS-1912037, DMS-1953244. K. Campbell, H. Dai, S. Joshi
  and P. Fletcher were supported by NSF grant DMS-1912030. Z. Su was supported by NSF grant
  DMS-1912037, NIH/NIAAA award R01-AA026834. Data were provided in part by the Human Connectome
  Project, WU-Minn Consortium (Principal Investigators: David Van Essen and Kamil Ugurbil;
  1U54MH091657) funded by the 16 NIH Institutes and Centers that support the NIH Blueprint for
  Neuroscience Research; and by the McDonnell Center for Systems Neuroscience at Washington
  University.}

\ethics{The work follows appropriate ethical standards in conducting research and writing the
  manuscript, following all applicable laws and regulations regarding treatment of animals or human
  subjects.}

\coi{The authors have no conflicts of interest.}

\bibliography{main}

\end{document}